\documentclass{article}
\usepackage{jmlr2e}
\usepackage{natbib}
\usepackage{times}

\usepackage[utf8]{inputenc} 
\usepackage[T1]{fontenc}    
\usepackage{booktabs}       
\usepackage{nicefrac}       
\usepackage{microtype}      
\usepackage{amsmath,amssymb,amsthm,amsfonts,color,bbm,float,graphicx,mathtools}
\usepackage[font=footnotesize]{caption}
\usepackage[caption = false]{subfig}
\usepackage{fancyhdr,paralist,array}
\usepackage{wrapfig,pdfpages}
\usepackage{physics}

\renewcommand{\b}{\mathbb}

\newcommand{\bR}{\mathbb{R}}

\newcommand{\Rmn}{\mathbb{R}^{n\times n}}
\DeclareMathOperator*{\argmin}{argmin}

\newcommand{\innerprod}[2]{\langle #1,#2 \rangle}

\renewcommand{\hat}{\widehat}

\newcommand{\set}[1]{\left\{#1\right\}}

\newcommand{\removed}[1]{}

\newtheorem*{conjecture*}{Conjecture}

\newtheorem{theorem}{Theorem}

\newtheorem{corollary}[theorem]{Corollary}

\begin{document}
\title{Implicit Regularization in Matrix Factorization}
\author{\name Suriya Gunasekar \email suriya@ttic.edu \AND \name Blake Woodworth \email blake@ttic.edu \AND \name Srinadh Bhojanapalli \email srinadh@ttic.edu \AND Behnam Neyshabur \email behnam@ttic.edu \AND Nathan Srebro \email nati@ttic.edu}

\maketitle
\begin{abstract}
We study implicit regularization when optimizing an underdetermined quadratic objective over a matrix $X$ with gradient descent on a factorization of $X$.  We conjecture and provide empirical and theoretical evidence that with small enough step sizes and initialization close enough to the origin, gradient descent on a full dimensional factorization converges to the minimum nuclear norm solution.
\end{abstract}

\section{Introduction}
When optimizing underdetermined problems with multiple global minima, the choice of optimization algorithm can play a crucial role in biasing us toward a specific global minima, even though this bias is not explicitly specified in the objective or problem formulation.  For example, using gradient descent to optimize an unregularized, underdetermined least squares problem would yield the minimum Euclidean norm solution, while using coordinate descent or preconditioned gradient descent might yield a different solution. Such implicit bias, which can also be viewed as a form of regularization, can play an important role in learning.  

In particular, implicit regularization has been shown to play a crucial role in training deep models \citep{neyshabur2015search,neyshabur2017geometry,zhang2017understanding,keskar2017large}: deep models often generalize well even when trained purely by minimizing the training error without any explicit regularization, and when there are more parameters than samples and the optimization problem is underdetermined.  Consequently, there are many zero training error solutions, all global minima of the training objective, some of which my generalize horribly.
Nevertheless, our choice of optimization algorithm, typically a variant of gradient descent, seems to prefer solutions that do generalize well.  This generalization ability cannot be explained by the capacity of the explicitly specified model class (namely, the functions representable in the chosen architecture).  Instead, it seems that the optimization algorithm biases us toward a ``simple" model, minimizing some implicit “regularization measure”, and that generalization is linked to this measure.  But what are the regularization measures that are implicitly minimized by different optimization procedures?

As a first step toward understanding implicit regularization in complex models, in this paper we carefully analyze implicit regularization in matrix factorization models, which can be viewed as two-layer networks with linear transfer.  We consider gradient descent on the entries of the factor matrices, which is analogous to gradient descent on the weights of a multilayer network.  We show how such an optimization approach can indeed yield good generalization properties even when the problem is underdetermined. We identify the implicit regularizer as the {\em nuclear norm}, and show that even when we use a full dimensional factorization, imposing no constraints on the factored matrix, optimization by gradient descent on the factorization biases us toward the minimum nuclear norm solution.  Our empirical study leads us to conjecture that with small step sizes and initialization close to zero, gradient descent converges to the minimum nuclear norm solution, and we provide empirical and theoretical evidence for this conjecture, proving it in certain restricted settings.

\section{Factorized Gradient Descent for Matrix Regression}\label{sec:gd}
We consider least squares objectives over matrices $X\in\bR^{n\times
  n}$ of the form:
\begin{equation}
\min_{X \succeq 0} F(X)=\norm{\mathcal{A}(X)-y}_2^2. 
\label{eq:lstsq}
\end{equation}
where $\mathcal{A}:\bR^{n\times n}\to \b{R}^m$ is a linear operator
specified by $\mathcal{A}(X)_i=\innerprod{A_i}{X}$, $A_i\in\bR^{n\times n}$,
and $y\in\bR^m$.  Without loss of generality, we consider only
symmetric positive semidefinite (p.s.d.) $X$ and symmetric linearly
independent $A_i$ (otherwise, consider optimization over a larger matrix $\begin{bsmallmatrix}W & X\\ X^\top& Z\end{bsmallmatrix}$ with $\mathcal{A}$
operating symmetrically  on the off-diagonal blocks).  
In particular, this setting covers
problems including matrix completion (where $A_i$ are indicators,  
\citet{candes2009exact}), matrix reconstruction from linear measurements \citep{recht2010guaranteed} and
multi-task training (where each column of $X$ is a predictor for a
deferent task and $A_i$ have a single non-zero column, \citet{argyriou2007multi,amit2007uncovering}).

We are particularly interested in the regime where $m \ll n^2$, in
which case \eqref{eq:lstsq} is an underdetermined system
with many global minima satisfying $\mathcal{A}(X)=y$.  For such
underdetermined problems, merely minimizing \eqref{eq:lstsq} cannot
ensure recovery (in matrix completion or recovery problems) or
generalization (in prediction problems).  For example, in a matrix
completion problem (without diagonal observations), we can minimize
\eqref{eq:lstsq} by setting all non-diagonal unobserved entries to zero,
or to any other arbitrary value.

Instead of working on $X$ directly, we will study a factorization
$X=UU^\top$. We can write \eqref{eq:lstsq} equivalently as optimization over $U$ as,
\begin{equation}
\min_{U\in\bR^{n\times d}} f(U)=\norm{\mathcal{A}(UU^\top)-y}_2^2. 
\label{eq:lstsq_u}
\end{equation}
When $d<n$, this imposes a constraint on the rank of $X$, but we will
be mostly interested in the case $d=n$, under which  no additional
constraint is imposed on $X$ (beyond being p.s.d.) and
\eqref{eq:lstsq_u} is equivalent to \eqref{eq:lstsq}. Thus, if $m \ll n^2$,  then \eqref{eq:lstsq_u} with $d=n$ is
similarly underdetermined and can be optimized in many ways --- estimating a global optima cannot ensure
generalization (e.g.~imputing zeros in a matrix completion
objective). Let us investigate what happens when we optimize
\eqref{eq:lstsq_u} by gradient descent on $U$.  

To simulate such a matrix reconstruction problem, we generated $m \ll n^2$
random measurement matrices and set $y = \mathcal{A}({X^*})$ according to
some planted $X^*\succeq 0$. We minimized \eqref{eq:lstsq_u} by
performing gradient descent on $U$ to convergence, and then measured
the relative reconstruction error $\norm{X-X^*}_F$.  Figure
\ref{fig:test_err_gauss} shows the normalized training objective and
reconstruction error as a function of the dimensionality $d$ of the
factorization,  for different initialization and step-size policies,
and three different planted $X^*$.
\begin{figure}[H]
\centering
\noindent
    \includegraphics[width=\textwidth]{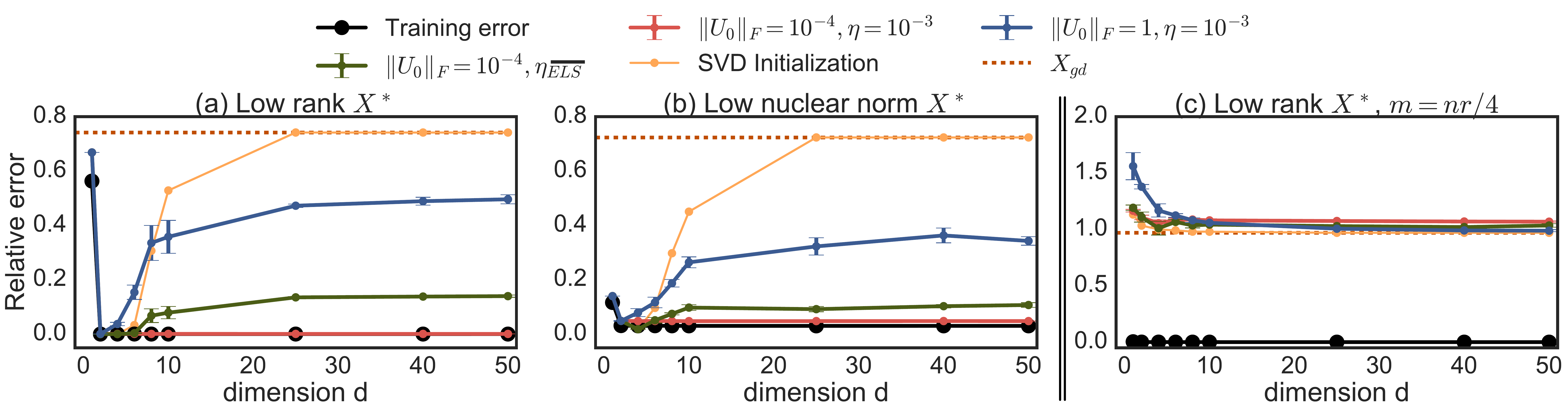}
   {\small \caption{ Reconstruction error of the solutions for the planted
      $50 \times 50$ matrix reconstruction problem.  In $(a)$ $X^*$ is of rank $r=2$ and $m=3nr$,  in $(b)$
      $X^*$ has a spectrum decaying as $O(1/k^{1.5})$ normalized to have $\|X^*\|_*=\sqrt{r}\norm{X^*}_F$ for $r=2$ and $m=3nr$, and in $(c)$ we look at a non-reconstructable setting where the number of measurements $m=nr/4$ is much smaller than the requirement to reconstruct a rank $r=2$  matrix. The plots compare the reconstruction error of gradient descent on $U$ for different choices initialization $U_0$ and step size  $\eta$, including fixed step-size and exact line search clipped for stability  ($\eta_{\overline{ELS}}$). Additonally,  the orange dashed reference line represents the performance of $X_{gd}$ --- a rank
      unconstrained global optima obtained by projected gradient descent on $X$ space for \eqref{eq:lstsq}, and `SVD-Initialization' is  an example of an alternate  rank $d$ global optima, where  initialization $U_0$ is picked based on SVD of $X_{gd}$ and gradient descent  with small stepsize is run on factor space. The results are averaged across $3$ random initialization and (nearly zero) errorbars indicate the standard deviation. 
        } 
    \label{fig:test_err_gauss}}
\end{figure} 

First, we see that (for sufficiently large $d$) gradient descent 
indeed finds a global optimum, as evidenced by the training
error (the optimization objective) being zero.  
This is not surprising since
with large enough $d$ this non-convex problem has no spurious local
minima \citep{burer2003nonlinear,journee2010low}  and
gradient descent converges almost surely to a global optima \citep{lee2016gradient}; there has also been recent work
establishing conditions for global convergence for low $d$ \citep{bhojanapalli2016global,ge2016matrix}.

The more surprising observation is that in panels $(a)$ and $(b)$, even when $d>m/n$, indeed even
for $d=n$, we still get good 
reconstructions from the solution of gradient descent with  initialization 
$U_0$ close to zero and small step size.  In this regime, \eqref{eq:lstsq_u} is
underdetermined and minimizing it does not ensure
generalization. To emphasize this, we plot the reference behavior
  of a rank unconstrained global minimizer $X_{gd}$ obtained via projected gradient
  descent for \eqref{eq:lstsq} on the $X$ space. For $d<n$ we also plot an example of an alternate ``bad" 
        rank $d$ global optima obtained with an initialization based on SVD of $X_{gd}$ (`{SVD Initialization}').
        
         When $d<m/n$, we understand how
the low-rank structure can guarantee generalization
\citep{srebro2005generalization} and reconstruction
\citep{keshavan2012efficient,bhojanapalli2016global,ge2016matrix}.
What ensures generalization when $d\gg m/n$?  
  Is there a strong
implicit regularization at play for the case of gradient descent on factor space and initialization close to zero?

\begin{figure}[htb]
\centering
\noindent
\includegraphics[width=\textwidth]{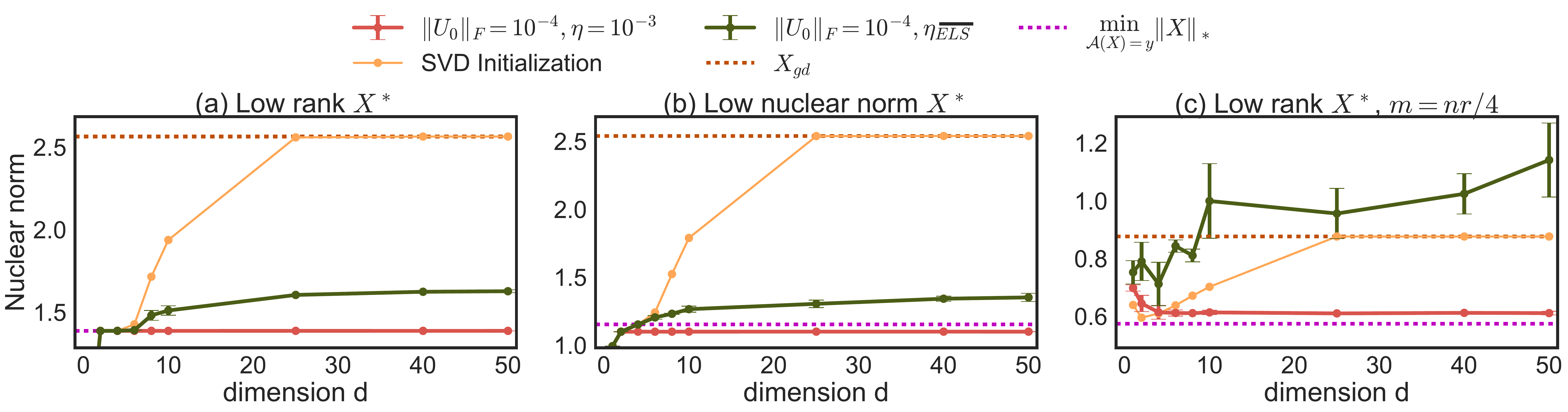}
{\small \caption{Nuclear norm of the solutions from Figure
  \ref{fig:test_err_gauss}. 
  In addition to the reference of $X_{gd}$ from Figure~\ref{fig:test_err_gauss}, the  magenta dashed line (almost overlapped by the plot of $\|U\|_F=10^{-4},\eta=10^{-3}$) is added as a reference for the (rank unconstrained)  minimum nuclear norm global optima. 
  The error bars indicate the standard deviation across $3$ random initializations. We have dropped the plot for $\norm{U}_F=1,\eta=10^{-3}$ to reduce clutter. 
  }
\label{fig:NN_gauss}}
\end{figure} 

Observing the nuclear norm of the resulting solutions plotted in
Figure \ref{fig:NN_gauss} suggests that gradient descent implicitly
induces a low nuclear norm solution.  This is the case even for $d=n$
when the factorization imposes no explicit constraints. Furthermore,
we do not include any explicit regularization and optimization is run
to convergence without any early stopping.  In fact, we can 
see a clear bias toward low nuclear norm even in problems where
reconstruction is not possible: in panel (c) of Figure
\ref{fig:NN_gauss} the number of samples $m=nr/4$ is much smaller than those required to reconstruct a rank $r$ ground truth matrix $X^*$.
The optimization in \eqref{eq:lstsq_u} is highly underdetermined and there are many possible zero-error global minima, but  
gradient descent still prefers a
lower nuclear norm solution.  The emerging story is that gradient
descent biases us to a low nuclear norm solution, and we already know
how having low nuclear norm can ensure generalization
\citep{srebro2005rank,foygel2011concentration} and minimizing the
nuclear norm ensures reconstruction
\citep{recht2010guaranteed,candes2009exact}.

Can we more explicitly characterize this bias?  We see that we do
not always converge precisely to the minimum nuclear norm solution.
In particular, the choice of step size and initialization affects
which solution gradient descent converges to.
Nevertheless, as we formalize in Section \ref{sec:gf}, we argue that
when $U$ is full dimensional, the step size becomes small enough, and
the initialization approaches zero, gradient descent will converge
precisely to a minimum nuclear norm solution, i.e.~to
$\argmin_{X\succeq0} \norm{X}_*\text{ s.t. } \mathcal{A}(X)=y$.

\section{Gradient Flow and Main Conjecture} \label{sec:gf}
 
The behavior of gradient descent with infinitesimally small step size
is captured by the differential equation $\dot{U}_t := \dv{U_t}{t} =
-\nabla f(U_t)$ with an initial condition for $U_0$.  For the optimization in 
\eqref{eq:lstsq_u} this is
\begin{equation} \label{eq:gradflowU}
\dot{U}_t = -\mathcal{A}^*(\mathcal{A}(U_tU_t^\top)-y)U_t,
\end{equation}
where $\mathcal{A}^*:\b{R}^m\to\bR^{n\times n}$ is the adjoint of $\mathcal{A}$
and is given by $\mathcal{A}^*(r)=\sum_ir_iA_i$. Gradient descent can be
seen as a discretization of \eqref{eq:gradflowU}, and approaches
\eqref{eq:gradflowU} as the step size goes to zero.

The dynamics \eqref{eq:gradflowU} define the behavior of the solution
$X_t=U_t U_t^\top$ and using the chain rule we can verify that
$\dot{X}_t = \dot{U}_tU_t^\top + U_t\dot{U}_t^\top =
-\mathcal{A}^*(r_t)X_t - X_t\mathcal{A}^*(r_t)$, where $r_t =
\mathcal{A}(X_t)-y$ is a vector of  the residual.  That is, even though the dynamics
are defined in terms of specific factorization $X_t=U_t U_t^\top$,
they are actually independent of the factorization and can be
equivalently characterized as
\begin{equation} \label{eq:gradflowX}
\dot{X}_t = -\mathcal{A}^*(r_t)X_t - X_t\mathcal{A}^*(r_t).
\end{equation}
We can now define the limit point  $X_\infty(X_{\textrm{init}}) := \lim_{t\to\infty}
X_t$ for the factorized gradient flow  \eqref{eq:gradflowX} initialized
at $X_0=X_{\textrm{init}}$.  We emphasize that these dynamics are very different from
the standard gradient flow dynamics of \eqref{eq:lstsq} on $X$, corresponding to gradient
descent on $X$, which take the form $\dot{X}_t = -\nabla F(X_t) =
-\mathcal{A}^*(r_t)$. 

Based on the preliminary experiments in Section~\ref{sec:gd} and a
more comprehensive numerical study discussed in Section~\ref{sec:exp},
we state our main conjecture as follows:
\begin{conjecture*}
  For any full rank $X_{\textrm{init}}$, if $\hat{X}=\lim_{\alpha\to 0}
  X_\infty(\alpha X_{\textrm{init}})$ exists and is a global optima for
  \eqref{eq:lstsq} with $\mathcal{A}(\hat{X})=y$, then $\hat{X} \in
  \argmin_{X\succeq 0}\ \norm{X}_*\ \textrm{s.t.}\
  \mathcal{A}(X) = y$.
\end{conjecture*}
Requiring a full-rank initial point demands a full dimensional
$d=n$ factorization in \eqref{eq:lstsq_u}.  The assumption of global
optimality in the conjecture is generally satisfied: for almost all
initializations, gradient flow will converge to a local minimizer
\citep{lee2016gradient}, and when $d=n$ any such local minimizer is also
global minimum \citep{journee2010low}.  Since we are primarily concerned with
underdetermined problems, we expect the global optimum to achieve
zero error, i.e.~satisfy $\mathcal{A}(X)=y$.  We already know from these
existing literature that gradient descent (or gradient flow) will
generally converge to \emph{a} solution satisfying $\mathcal{A}(X)=y$; the question we address
here is \emph{which} of those solutions will it converge to.  


The conjecture implies the same behavior for asymmetric
problems factorized as $X = UV^\top$ with gradient flow on $(U,V)$,
since this is equivalent to gradient flow on the 
p.s.d. factorization of $\left[\begin{smallmatrix}W&X\\X^\top&Z\end{smallmatrix}\right]$.

\section{Theoretical Analysis}\label{sec:theory}

We will prove our conjecture for the special case where the matrices
$A_i$ commute, and discuss the more challenging non-commutative case.
But first, let us begin by reviewing the behavior of
straight-forward gradient descent on $X$ for the convex problem in \eqref{eq:lstsq}.

\paragraph{Warm up: } Consider gradient descent updates on the
original problem \eqref{eq:lstsq} in $X$ space, ignoring the
p.s.d.~constraint.
The gradient direction $\nabla
F(X) = \mathcal{A}^*(\mathcal{A}(X)-y)$ is always spanned by the $m$ 
matrices $A_i$.  Initializing at $X_{\textrm{init}}=0$, we will therefore always
remain in the $m$-dimensional subspace $\mathcal{L} = \left\{ X=\mathcal{A}^*(s)
  \middle| s\in\bR^m \right\}$.  Now consider the optimization problem
$\min_X \norm{X}^2_F\ \text{s.t.}\ \mathcal{A}(X)=y$.  The KKT optimality
conditions for this problem are $\mathcal{A}(X)=y$  and $\exists \nu$ s.t. 
$X=\mathcal{A}^*(\nu)$.  As long as we are in $\mathcal{L}$, the second condition
is satisfied, and if we converge to a zero-error global minimum, then
the first condition is also satisfied.  Since gradient descent stays
on this manifold, this establishes that if gradient descent converges
to a zero-error solution, it is the minimum Frobenius norm solution.

\paragraph{Getting started: $\mathbf{m=1}$} Consider the simplest case of
the factorized problem when $m=1$ with $A_1=A$ and $y_1=y$. The
dynamics of \eqref{eq:gradflowX} are given by $\dot{X}_t =- r_t (A X_t
+ X_tA)$, where $r_t$ is simply a scalar, and the solution for $X_t$
is given by, $X_t  = \exp\left( s_t A \right)X_0\exp\left( s_t A \right)$
where $s_T = -\int_0^T r_tdt$. Assuming  $\hat{X} =
\lim_{\alpha\to0}X_\infty (\alpha X_0)$ exists and $\mathcal{A}(\hat{X}) = y$,
 we want to show $\hat{X}$ is an optimum for the
following problem
\begin{equation}\label{eq:minNN}
\min_{X\succeq 0} \norm{X}_*\ \ \textrm{s.t.}\ \ \mathcal{A}(X) = y.
\end{equation}
The KKT optimality conditions for \eqref{eq:minNN} are:
\begin{equation} \label{eq:KKT} 
\exists \nu \in \mathbb{R}^m\ \textrm{s.t.}
\qquad \mathcal{A}(X) = y 
\qquad X \succeq 0
\qquad \mathcal{A}^*(\nu) \preceq I
\qquad (I - \mathcal{A}^*(\nu))X = 0
\end{equation}
We already know that the first condition holds, and the p.s.d.~condition is 
guaranteed by the factorization of $X$. The remaining complementary slackness and dual feasibility  
conditions effectively require that $\hat{X}$ is spanned
by the top eigenvector(s) of $A$. Informally, looking to the gradient flow path
above, for any non-zero $y$, as $\alpha\to 0$ it is necessary that
$|s_\infty|\to \infty$ in order to converge to a global optima, thus
eigenvectors corresponding to the top eigenvalues of $A$ will dominate
the span of $X_\infty(\alpha X_{\textrm{init}})$.

\paragraph{What we can prove: Commutative $\mathbf{\set{A_i}_{i\in[m]}}$}
The characterization of the the gradient flow path from the previous
section can be extended to arbitrary $m$ in the case that the matrices
$A_i$ commute, i.e.~$A_iA_j = A_jA_i$ for all $i,j$. Defining
$s_T = -\int_0^T r_tdt$ -- a vector integral, we can verify by
differentiating that solution of  \eqref{eq:gradflowX} is 
\begin{equation} \label{eq:commutativepath}
X_t = \exp\left( \mathcal{A}^*(s_t)\right)X_0\exp\left( \mathcal{A}^*(s_t) \right)
\end{equation}


\begin{theorem}\label{thm:commutative}
In the case where matrices $\set{A_i}_{i=1}^m$ commute,  if  $\hat{X} = \lim_{\alpha\to 0}X_\infty(\alpha I)$ exists and is a global optimum for \eqref{eq:lstsq} with $\mathcal{A}(\hat{X}) = y$,  then $\hat{X} \in \argmin_{X\succeq0} \norm{X}_*\ \textrm{s.t.}\ \mathcal{A}(X) = y$.
\end{theorem}
\begin{proof}
It suffices to show that such a $\hat{X}$ satisfies the complementary slackness and dual feasibility KKT conditions in \eqref{eq:KKT}.
Since the matrices $A_i$ commute and are symmetric, they are simultaneously diagonalizable by a basis $v_1,..,v_n$, and so is $\mathcal{A}^*(s)$ for any $s \in \mathbb{R}^m$. This implies that for any $
\alpha$, $X_\infty(\alpha I)$ given by \eqref{eq:commutativepath} and its limit $\hat{X}$ also have the same eigenbasis. Furthermore, since $X_\infty(\alpha I)$ converges to $\hat{X}$, the scalars $v_k^\top X_\infty(\alpha I) v_k \to v_k^\top \hat{X} v_k$ for each $k \in [n]$. Therefore, $\lambda_k(X_\infty(\alpha I)) \to \lambda_k(\hat{X})$, where $\lambda_k(\cdot)$ is defined as the eigenvalue corresponding to eigenvector $v_k$ and \emph{not} necessarily the $k^{\textrm{th}}$ largest eigenvalue.

Let $\beta = -\log \alpha$, then $\lambda_k(X_\infty(\alpha I)) = \exp(2\lambda_k(\mathcal{A}^*(s_\infty(\beta))) - 2\beta)$. For all $k$ such that $\lambda_k(\hat{X}) > 0$, by the continuity of $\log$, we have 
\vspace{-10pt}
\begin{equation}
2\lambda_k(\mathcal{A}^*(s_\infty(\beta))) - 2\beta - \log \lambda_k(\hat{X}) \to 0 \implies \lambda_k\Big(\mathcal{A}^*\big(\frac{s_\infty(\beta)}{\beta}\big)\Big) - 1 - \frac{\log \lambda_k(\hat{X})}{2\beta} \to 0.
\end{equation}
Defining $\nu(\beta) = \nicefrac{s_\infty(\beta)}{\beta}$, we conclude that for all $k$ such that $\lambda_k(\hat{X})\neq0$,  $\lim_{\beta\to\infty}\lambda_k(\mathcal{A}^*(\nu(\beta))) = 1$.
Similarly, for each $k$ such that $\lambda_k(\hat{X}) = 0$,
\begin{equation}
\exp(2\lambda_k(\mathcal{A}^*(s_\infty(\beta))) - 2\beta) \to 0 \implies \exp(\lambda_k(\mathcal{A}^*(\nu(\beta))) - 1)^{2\beta} \to 0.
\end{equation}
Thus, for every $\epsilon\in(0,1]$, for sufficiently large $\beta$
\begin{equation}
\exp(\lambda_k(\mathcal{A}^*(\nu(\beta))) - 1) < \epsilon^{\frac{1}{2\beta}} < 1 \implies \lambda_k(\mathcal{A}^*(\nu(\beta))) < 1.
\end{equation}
Therefore, we have shown that $\lim_{\beta\to\infty} \mathcal{A}^*(\nu(\beta)) \preceq I$ and $\lim_{\beta\to\infty} \mathcal{A}^*(\nu(\beta))\hat{X} = \hat{X}$ establishing the optimality of $\hat{X}$ for \eqref{eq:minNN}.
\end{proof}
Interestingly, and similarly to gradient descent on $X$, this proof does
not exploit the particular form of the ``control" $r_t$ and only
relies on the fact that the gradient flow path stays within the
manifold
\begin{equation}\label{eq:commutativemanifold}
 \mathcal{M} = \set{X = \exp\left( \mathcal{A}^*(s)
  \right)X_{\textrm{init}}\exp\left( \mathcal{A}^*(s) \right)\ \middle|\ s \in
  \mathbb{R}^m}.
\end{equation} 
Since the $A_i$'s commute, we can verify that the
tangent space of $\mathcal{M}$ at a point $X$ is given by $T_X\mathcal{M} =
\textrm{Span}\set{A_iX + XA_i}_{i\in[m]}$, thus gradient flow will
always remain in $\mathcal{M}$. 
For any control $r_t$ such that following
$\dot{X}_t = -\mathcal{A}^*(r_t)X_t - X_t\mathcal{A}^*(r_t)$ leads to a zero error
global optimum, that optimum will be a minimum nuclear norm solution.
This implies in particular that the conjecture extends to gradient
flow on \eqref{eq:lstsq_u} even when the Euclidean norm is replaced by
certain other norms, or when only a subset of measurements are used
for each step (such as in stochastic gradient descent).

However, unlike gradient descent on $X$, the manifold $\mathcal{M}$ is not
flat, and the tangent space at each point is different.  Taking
finite length steps, as in gradient descent, would cause us to ``fall off" of the
manifold.  To avoid this, we must take
infinitesimal steps, as in the gradient flow dynamics.

In the case that $X_{\textrm{init}}$ and the measurements $A_i$ are diagonal
matrices, gradient descent on \eqref{eq:lstsq_u} is equivalent to a
vector least squares problem, parametrized in terms of the square root
of entries:
\begin{corollary}
  Let $x_\infty(x_{\textrm{init}})$ be the limit point of gradient flow on
  $\min_{u\in\mathbb{R}^n} \norm{A x(u)- y}_2^2$ with initialization $x_{\textrm{init}}$, where
  $x(u)_i=u_i^2$, $A\in\bR^{m \times n}$ and $y\in\bR^m$. If $\hat{x} = \lim_{\alpha \to 0} x_\infty(\alpha\vec{1})$
  exists and $A\hat{x} = y$, then $\hat{x}
  \in \argmin_{x\in\bR^m_+} \norm{x}_1\ \textrm{s.t.}\ Ax = y$.
\end{corollary}

\paragraph{The plot thickens: Non-commutative $\mathbf{\set{A_i}_{i\in[m]}}$}
Unfortunately, in the case that the matrices $A_i$ do not commute,
analysis is much more difficult. For a matrix-valued function $F$,
$\dv{}{t}\exp(F_t)$ is  equal to $\dot{F_t}\exp(F_t)$ only when
$\dot{F_t}$ and $F_t$ commute. Therefore, \eqref{eq:commutativepath} is no longer a valid
solution for \eqref{eq:gradflowX}.  Discretizing the solution path, we can
express the solution as the ``time ordered exponential":\vspace{-5pt}
\begin{equation} \label{eq:timeorderedexponential} X_t =
  \lim_{\epsilon\to0} \left(\prod_{\tau=t/\epsilon}^1 \exp\left(-
    \epsilon\mathcal{A}^*(r_{\tau\epsilon}) \right)
  \right)X_0\left(\prod_{\tau=1}^{t/\epsilon} \exp\left(-
    \epsilon\mathcal{A}^*(r_{\tau\epsilon}) \right)\right),\vspace{-5pt}
\end{equation}
where the order in the products is important.  If  $A_i$
commute, the product of exponentials is equal to an exponential of
sums, which in the limit evaluates to  
  the solution in \eqref{eq:commutativepath}. However, since in general $\exp(A_1)\exp(A_2) \neq
\exp(A_1+A_2)$,
the path \eqref{eq:timeorderedexponential} is \emph{not} contained in the
manifold $\mathcal{M}$ defined in \eqref{eq:commutativemanifold}.

It is tempting to try to construct a new manifold $\mathcal{M}'$ such that
$\textrm{Span}\set{A_iX + XA_i}_{i\in[m]}
\subseteq T_X\mathcal{M}'$ and $X_0 \in \mathcal{M}'$, ensuring the gradient flow remains in
$\mathcal{M}'$. However, since  $A_i$'s do not commute, by
combining infinitesimal steps along different directions, it is
possible to move (very slowly) in directions that are \emph{not} of
the form $\mathcal{A}^*(s)X + X\mathcal{A}^*(s)$ for any $s \in \mathbb{R}^m$. The
possible directions of movements indeed corresponds to the Lie algebra defined
by the closure of $\set{A_i}_{i=1}^m$ under the commutator operator $[A_i,A_j]
:= A_iA_j - A_jA_i$. Even when $m=2$, this closure will generally
encompass \emph{all} of $\Rmn$, allowing us to approach any
p.s.d.~matrix $X$ with some (wild) control $r_t$. Thus, we
cannot hope to ensure the KKT conditions for an arbitrary control as we
did in the commutative case --- it is necessary to exploit the structure
of the residuals $\mathcal{A}(X_t) - y$ in some way.

Nevertheless, in order to make finite progress moving along a commutator direction like $[A_i,A_j]X_t +
X_t[A_i,A_j]^\top$, it is necessary to use an extremely non-smooth
control, e.g., looping $1/\epsilon^2$ times between $\epsilon$ steps in the
directions $A_i,A_j,-A_i,-A_j$, each such loop making an $\epsilon^2$
step in the desired direction.  We expect the actual residuals $r_t$ to
behave much more smoothly and that for smooth control the
non-commutative terms in the expansion of the time ordered exponential
\eqref{eq:timeorderedexponential} are asymptotically lower order then the direct term
$\mathcal{A}^*(s)$ (as $X_{\textrm{init}} \rightarrow 0$).  This is indeed confirmed
numerically, both for the actual residual controls of the gradient
flow path, and for other random controls.


\section{Empirical Evidence} \label{sec:exp} 
Beyond the matrix reconstruction experiments of Section \ref{sec:gd},
we also conducted  experiments with similarly simulated matrix completion problems,
including problems where entries are sampled from power-law distributions (thus not satisfying incoherence), 
as well as matrix completion problem on non-simulated Movielens data. In addition to gradient descent, we also looked
more directly at the gradient flow ODE \eqref{eq:gradflowU} and used a numerical  ODE solver provided as part of \texttt{SciPy} \citep{scipy}. But we still uses a finite (non-zero) initialization.  We also emulated staying 
on a valid ``steering path" by numerically approximating the
time ordered exponential of \ref{eq:timeorderedexponential} --- for a finite
discretization $\eta$,  instead of moving linearly
in the direction of the gradient $\nabla f(U)$ (like in gradient descent), we multiply $X_t$ on right and left by $ e^{-\eta \mathcal{A}^*(r_t)}$.
The results of these experiments are summarized in Figure~\ref{fig:matrix_completion}.

\renewcommand{\thesubfigure}{\roman{subfigure}}
\begin{figure}[H]
\centering
\subfloat[Gaussian random measurements. We report the nuclear norm of the gradient flow solutions from three different approximations to \eqref{eq:gradflowU} -- numerical ODE solver (\textit{ODE approx.}),  time ordered exponential specified in \eqref{eq:timeorderedexponential} (\textit{Time ordered exp.}) and standard gradient descent with small step size (\textit{Gradient descent}). The nuclear norm of  the  solution from gradient descent on $X$ space -- $X_{gd}$ and the minimum nuclear norm global minima are provided as references.  In $(a)$ $X^*$ is rank $r$ and $m=3nr$, in $(b)$ $X^*$ has a decaying spectrum with $\|X^*\|_*=\sqrt{r}\|X^*\|_F$ and $m=3nr$, and in $(c)$ $X^*$ is rank $r$ with $m=nr/4$, where $n=50$, $r=2$.]
{\begin{minipage}[c]{\textwidth}
\centering
\includegraphics[width=0.8\linewidth]{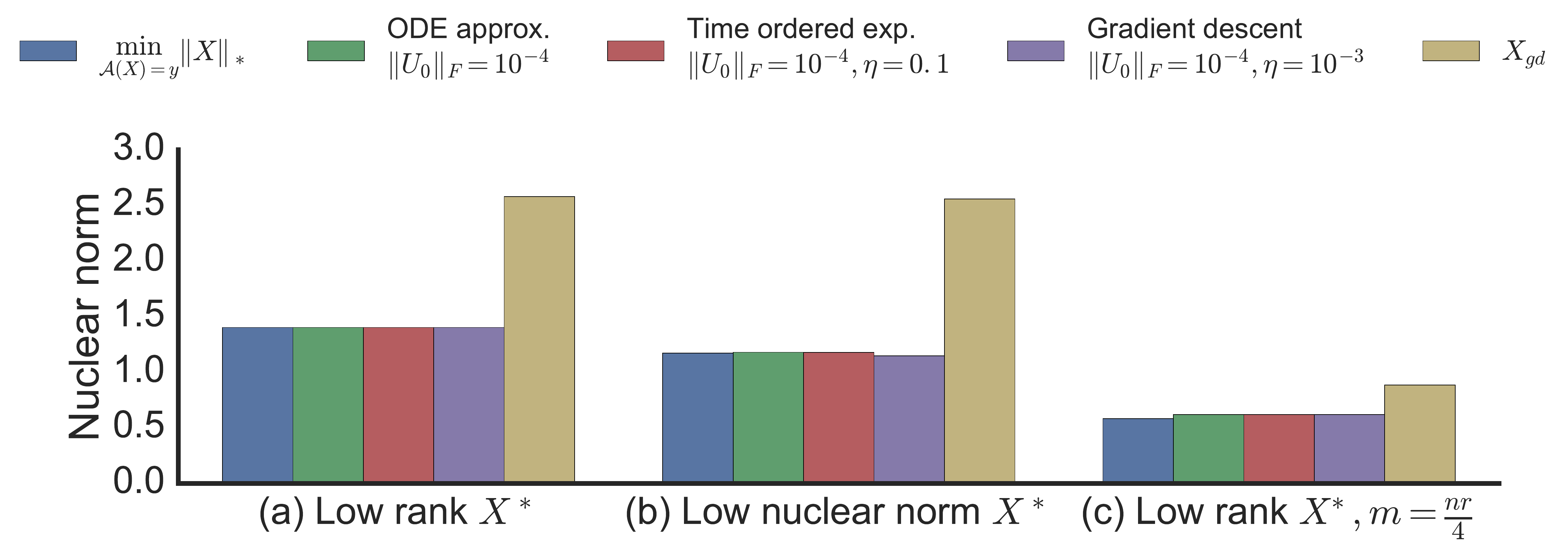}\label{fig:gauss_flow}
\end{minipage}}\\
\subfloat[Uniform matrix completion:  $\forall i$, $A_i$ measures a uniform random entry of $X^*$. Details on $X^*$,  number of measurements, and the legends follow  Figure\ref{fig:matrix_completion}-(i).]
{\begin{minipage}[c]{\textwidth}
\centering
\includegraphics[width=0.8\linewidth]{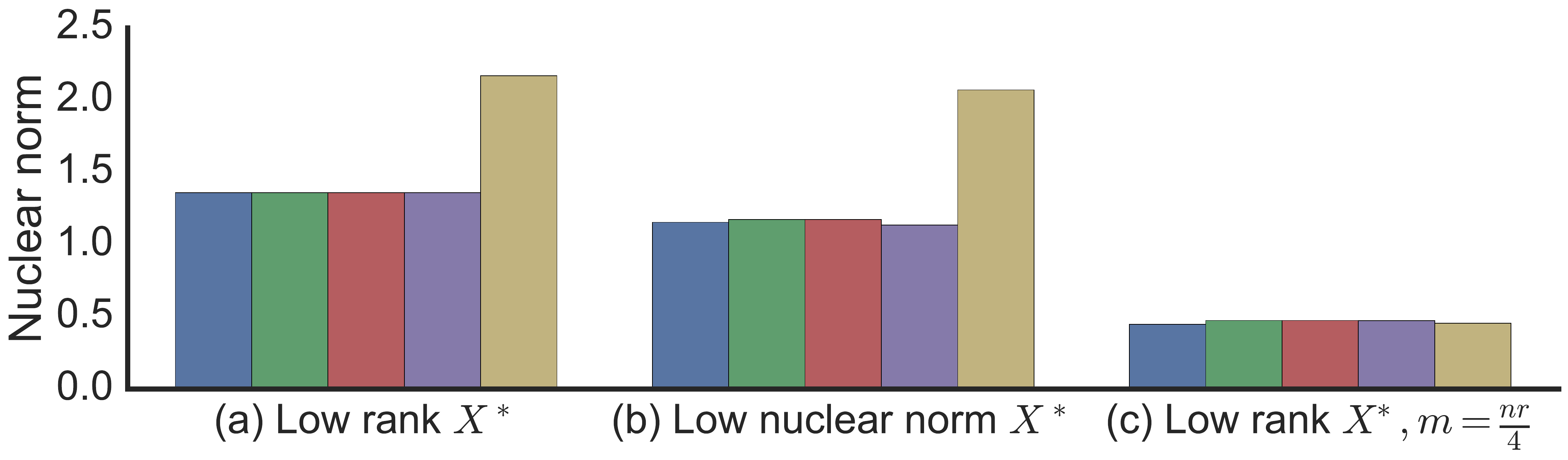}\label{fig:umc_flow}
\end{minipage}}\\
\subfloat[Power law matrix completion:  $\forall i$, $A_i$ measures a random entry of $X^*$  chosen according to a power law distribution. Details on $X^*$,  number of measurements, and the legends  follow  Figure\ref{fig:matrix_completion}-(i).]
{\begin{minipage}[c]{\textwidth}
\centering
\includegraphics[width=0.8\linewidth]{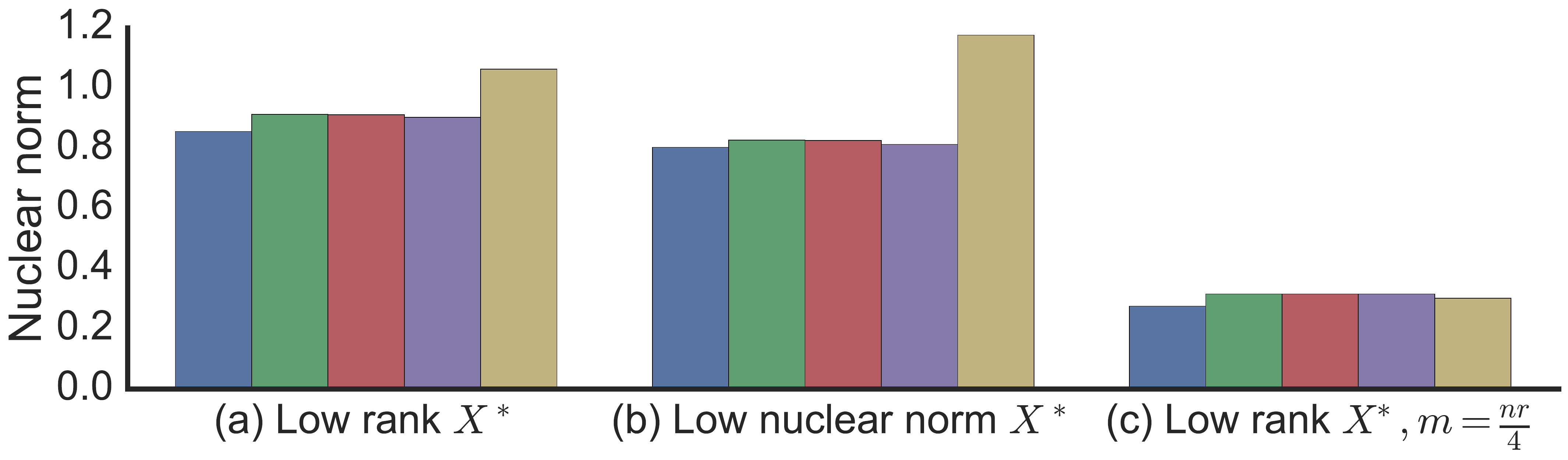}\label{fig:pmc_flow}
\end{minipage}}\\
\subfloat[Benchmark movie recommendation dataset --- Movielens $100$k. The dataset contains $\sim 100$k ratings from $n_1=943$ users on $n_2=1682$ movies. In this problem, gradient updates are performed on the asymmetric matrix factorization space $X=UV^\top$ with dimension $d=\min{(n_1,n_2)}$. 
The training data is completely fit to have $<\!10^{-2}$ error. Test error is computed on a held out data of $10$ ratings per user. Here we are not interested in the recommendation performance (test error) itself but on observing the bias of gradient flow with initialization close to zero to return a low nuclear norm solution --- the test error is provided merely to demonstrate the effectiveness of such a bias in this application. Also, due to the scale of the problem, we only report a coarse approximation of the gradient flow \ref{eq:gradflowU} from gradient descent with {$\|U_0\|_F=10^{-3}$, $\eta=10^{-2}$}.]
{\begin{minipage}[c]{\textwidth}
\centering
\begin{tabular}{|l|m{3.1cm}|m{3.1cm}|m{3.0cm}|}
\hline
            &  $\argmin_{\mathcal{A}(X)=y}\|X\|_*$ &  Gradient descent\newline  {\tiny$\|U_0\|_F=10^{-3}$, $\eta=10^{-2}$} &   $X_{gd}$ \\ \hline
Test Error   & $0.2880$ & $0.2631$ & $1.000$ \\ \hline
Nuclear norm & $8391$ & $8876$ & $20912$ \\ \hline
\end{tabular}
\label{tab:movielens}
\end{minipage}}
{\caption{Additional matrix reconstruction experiments}
\label{fig:matrix_completion}}
\end{figure}

In these experiments, we again observe trends similar to those in Section~\ref{sec:gd}.  In some panels in Figure~\ref{fig:matrix_completion}, we do see a discernible gap between the minimum nuclear norm global optima and the nuclear norm of the gradient flow solution with $\|U_0\|_F=10^{-4}$. This discrepancy could either be  due to starting at a non-limit point of $U_0$, or numerical issue arising from approximations to the ODE, or it could potentially suggest a weakening of the conjecture. Even if the later case were true, the experiments so far provide strong evidence for atleast approximate versions of our conjecture being true under a wide range of problems.

\paragraph*{Exhaustive search}
Finally, we also did experiments on an exhaustive grid search over small
problems, capturing essentially all possible problems of this size. We performed an exhaustive grid search for matrix completion problem
instances in symmetric p.s.d. ${3\times 3}$ matrices.  With $m=4$,
there are $15$ unique masks or $\{A_i\}_{i\in[4]}$'s that are valid
symmetric matrix completion observations. For each mask, we fill the
$m=4$ observations with all possible combinations of $10$ uniformly
spaced values in the interval $[-1,1]$. This gives us a total of
$15\times 10^4$ problem instances. Of these problems instances, we
discard the ones that do not have a valid PSD completion and run the
ODE solver on every
remaining instance with a random $U_0$ such that
$\|U_0\|_F=\bar{\alpha}$, for different values of $\bar{\alpha}$.
Results on the deviation from the minimum nuclear norm are reported in
Figure~\ref{fig:grid}.  For small $\bar{\alpha}=10^{-5}, 10^{-3}$,
most of instances of our grid search algorithm returned solutions with
near minimal nuclear norms, and the maximum deviation is within the
possibility of numerical error. This behavior also decays for
$\bar{\alpha}=1$.

\begin{figure}[H]
\centering
\includegraphics[width=\textwidth]{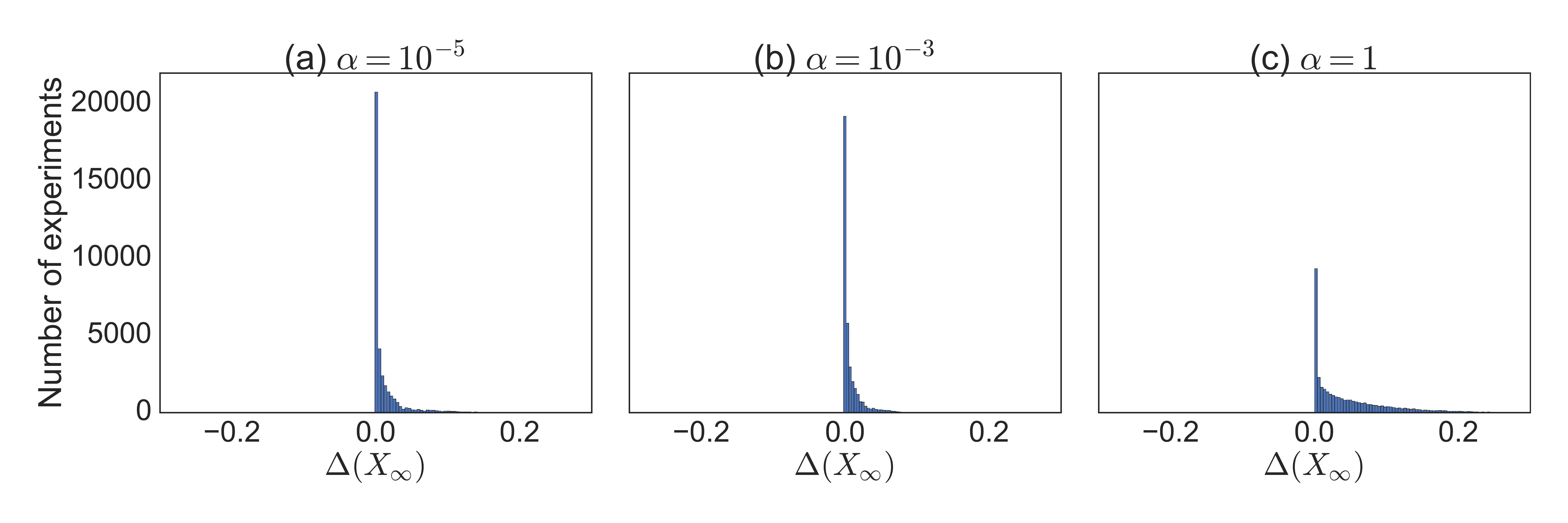}
{\small\caption{ Histogram of relative sub-optimality of nuclear norm of $X_\infty$ in grid search experiments. In this figure, we plot the histogram of $\Delta(X_\infty)=\frac{\norm{X_\infty}_*-\norm{X_\text{min}}_*}{\norm{X_\text{min}}_*}$, where $\norm{X_\text{min}}_*=\underset{\mathcal{A}(X)=y}{\min}\norm{X}_*$. The three panels correspond to different values of norm of initialization $\bar{\alpha}=\|U_0\|_F$. In  $(a)$ $\bar{\alpha}=10^{-5}$,  in  $(a)$ $\bar{\alpha}=10^{-3}$, and  in  $(c)$ $\bar{\alpha}=1$. \removed{In these figures, we had to restrict the limits of the plot to ensure good resolution.} }\label{fig:grid}}
\end{figure}


\section{Discussion}

It is becoming increasingly apparent that biases introduced by
optimization procedures, especially for under-determined problems, are
playing a key role in learning.  Yet, so far we have very little
understanding of the implicit biases associated with different
non-convex optimization methods.  In this paper we carefully study
such an implicit bias in a two-layer non-convex problem, identify it,
and show how even though there is no difference in the model class
(problems \eqref{eq:lstsq} and \eqref{eq:lstsq_u} are equivalent when
$d=n$, both with very high capacity), the non-convex modeling induces
a potentially much more useful implicit bias.

We also discuss how the bias in the non-convex case is much more
delicate then in convex gradient descent: since we are not restricted
to a flat manifold, the bias introduced by optimization depends on the
step sizes taken.  Furthermore, for linear least square problems
(i.e.~methods based on the gradients w.r.t.~$X$ in our formulation),
any global optimization method that uses linear combination of
gradients, including conjugate gradient descent, Nesterov acceleration
and momentum methods, remains on the manifold spanned by the
gradients, and so leads to the same minimum norm solution.  This is
not true if the manifold is curved, as using momentum or passed
gradients will lead us to ``shoot off'' the manifold.

Much of the recent work on non-convex optimization, and matrix
factorization in particular, has focused on global convergence:
whether, and how quickly, we converge to a global minima\removed{\cite{}}.  In
contrast, we address the complimentary question of {\em which} global
minima we converge to.  There has also been much work on methods
ensuring good matrix reconstruction or generalization based on
structural and statistical properties\removed{\cite{}}.  We do not
assume any such properties, nor that reconstruction is possible or
even that there is anything to reconstruct---for any problem of the
form \eqref{eq:lstsq} we conjecture that \eqref{eq:gradflowX} leads to the
minimum nuclear norm solution.  Whether such a minimum nuclear norm
solution is good for reconstruction or learning is a separate issue
already well addressed by the above literature.

We based our conjecture on extensive numerical simulations, with
random, skewed, reconstructible, non-reconstructible, incoherent,
non-incoherent, and and exhaustively enumerated problems, some of which is reported in Section \ref{sec:exp}.  We believe
our conjecture holds, perhaps with some additional technical
conditions or corrections.  We explain how the conjecture is related
to control on manifolds and the time ordered exponential and discuss a
possible approach for proving it.
\clearpage
\bibliography{low_rank_recovery,suriya}

\clearpage

\end{document}